\documentclass[twoside,11pt]{article}

\usepackage[abbrvbib, preprint]{jmlr2e}

\usepackage[utf8]{inputenc} 
\usepackage[T1]{fontenc}    
\usepackage{microtype}      
\usepackage{hyperref}       
\usepackage{booktabs}       
\usepackage{amsfonts}       
\usepackage{nicefrac}       
\usepackage{xcolor}         
\usepackage{url}            
\usepackage{subcaption}     
\usepackage{enumitem}  

\usepackage{placeins}  

\usepackage[most]{tcolorbox}

\newtcolorbox{greybox}{
  colback=black!5,      
  colframe=black!5,     
  arc=2pt,              
  boxrule=0pt,
  left=10pt, right=10pt, top=8pt, bottom=8pt,
  enhanced
}

\makeatletter
\let\proof\@undefined
\let\endproof\@undefined
\makeatother

\usepackage{amsmath,amsfonts,bm}

\usepackage{amsthm}
\usepackage{xcolor}
\usepackage[framemethod=TikZ]{mdframed}

\newmdtheoremenv[
    backgroundcolor=gray!10,
    linecolor=gray!10,
    linewidth=0pt,
    innertopmargin=8pt,
    innerbottommargin=8pt,
    innerleftmargin=10pt,
    innerrightmargin=10pt,
    roundcorner=2pt,
]{prop}{Proposition}

\def\eqref#1{equation~\ref{#1}}

\def\1{\bm{1}}

\DeclareMathAlphabet{\mathsfit}{\encodingdefault}{\sfdefault}{m}{sl}
\SetMathAlphabet{\mathsfit}{bold}{\encodingdefault}{\sfdefault}{bx}{n}

\usepackage{algorithmic}
\usepackage{algorithm}
\usepackage{graphicx}
\usepackage{booktabs}
\usepackage{enumitem}
\usepackage{wrapfig}
\usepackage{caption}
\usepackage{float}
\usepackage[capitalise,noabbrev]{cleveref}

\crefformat{equation}{(#2#1#3)}
\crefrangeformat{equation}{(#3#1#4) to (#5#2#6)}
\crefmultiformat{equation}{(#2#1#3)}{ and (#2#1#3)}{, (#2#1#3)}{ and (#2#1#3)}

\ShortHeadings{Annealing Flows with Surrogate Likelihood Estimators}{Daniel Peñaherrera, Rishal Aggarwal and David Ryan Koes}
\firstpageno{1}

\begin{document}

\title{Scalable Inference-Time Annealing with Surrogate Likelihood Estimators}

\author{\name Daniel Peñaherrera \email dap181@pitt.edu \\
       \name Rishal Aggarwal \email rishal.aggarwal@pitt.edu \\
       \name David Ryan Koes \email dkoes@pitt.edu \\
       \addr CMU-Pitt PhD Program in Computational Biology\\
        Dept.\ of Computational \& Systems Biology, University of Pittsburgh, Pittsburgh, PA 15260, USA}

\maketitle

\begin{abstract}%
    A long standing challenge in computational chemistry and biophysics is efficiently sampling the Boltzmann distribution of molecules. Advances in generative modeling have been proposed to address the limitations of conventional sampling techniques by eliminating the computational cost of simulation. A promising direction is iteratively finetuning diffusion models along a temperature ladder whereby training data is generated via importance sampling during inference-time annealing. Unfortunately, these methods require computing a divergence over the score field to estimate importance weights, rendering them intractable for larger systems. Here we present scalable inference-time annealing (SITA), which retrains flow-based models to generate samples at progressively lower temperatures using an energy-based model to facilitate fast surrogate likelihoods. We demonstrate state-of-the-art performance on both Alanine Dipeptide and Alanine Tripeptide while avoiding costly divergence terms. Our code is available at: \url{https://github.com/countrsignal/sita.git}
\end{abstract}


\section{Introduction}
\label{sec:intro}

Sampling the equilibrium ensemble of molecular configurations is a foundational task in statistical physics (\citet{noe_bg, H_nin_2022, Faulkner_2024}), as it provides access to thermodynamic observables like free energies and binding affinities, yet remains notoriously difficult for all but the simplest systems. Molecular ensembles are defined by the Boltzmann distribution whose probability density function is defined by $\pi(x) = Z^{-1}\exp(-\frac{E(x)}{k_{B}T})$ where $E(\cdot)$ is the energy function of the system $x$, $k_{B}$ is the Boltzmann constant, $T$ is the temperature, and $Z$ is the normalizing constant known as the partition function. The difficulty in sampling the Boltzmann distribution arises from the highly non-convex and rugged nature of the energy function. Traditional techniques such as Markov Chain Monte Carlo (MCMC) and molecular dynamics (MD) simulations often become trapped in energy minima, requiring long simulations with femtosecond timesteps that yield highly correlated samples. Furthermore, even minor modifications to the molecular system of interest necessitate entirely new simulations, limiting their suitability for high-throughput analysis.

Deep-learning generative models are an established class of samplers \citep{hyvarinen05a, sohldickstein2015, song2020score, ho2020denoising, lipman2022flow} that hold the promise for fast, amortized sampling of the Boltzmann distribution. However, the data-driven nature of these methods poses a severe limitation to realizing this promise as the current training paradigm is circular: it requires equilibrium ensembles for training data, yet generating such ensembles is precisely the intractable problem at hand. Recently, a new paradigm has emerged that facilitates a bootstrapping approach whereby generative models are iteratively retrained on their own outputs from a previous iteration of optimization. Existing generative bootstrapping methods can be broadly grouped into two families: those based on diffusion samplers and those based on importance sampling.

Diffusion-based samplers are designed to approximately sample a target distribution by defining probability paths to either optimize over \citep{ctmd} or derive regression targets for the optimal drift \citep{idem, asbs, bridge_matching}. In the absence of any initial dataset, these methods require incorporation of the unnormalized likelihoods under the target distribution into their loss functions. While compelling in principle, diffusion-based samplers tend to suffer from mode collapse when scaled to molecular systems. Although recent work has begun to address this \citep{bridge_matching}, the resulting methods still do not recover the equilibrium ensemble faithfully, trading mode collapse for an oversmoothed energy landscape.

Importance-sampling-based bootstrapping emulates Annealed Importance Sampling \citet{Jarzynski_1997, neal} and Sequential Monte Carlo \citet{smc} by defining a sequence of distributions that the model successively approximates. At each step, samples from the current model serve as proposals, reweighted by importance weights to target the next distribution. The sequence is typically constructed via geometric annealing \citep{fab, cmt_bless, nets} or temperature annealing \citep{temp_bg, akhound2025progressive}. Under temperature annealing, one trains a generative model on high-temperature simulation data and iteratively fine-tunes on its own samples as the temperature is lowered. Two benefits follow: (i) high-temperature simulation enhances exploration and mode coverage, and (ii) recursive bootstrapping yields data-efficient recovery of the target distribution.

One such annealing method is PITA \citet{akhound2025progressive}, which combines diffusion models with inference-time modifications to the reverse-time generative process, enabling generation of samples that are approximately distributed by the lower temperature Boltzmann distribution. PITA relies on self-normalized importance sampling (SNIS) at each annealing step based on the Feynman-Kac formula. This SNIS estimator, however, requires evaluating the divergence of the score field along the full integration path of the reverse process. 
Such computational overhead poses a serious limitation for systems with many degrees of freedom.

\begin{figure}[t]  
\centering
\includegraphics[width=\textwidth]{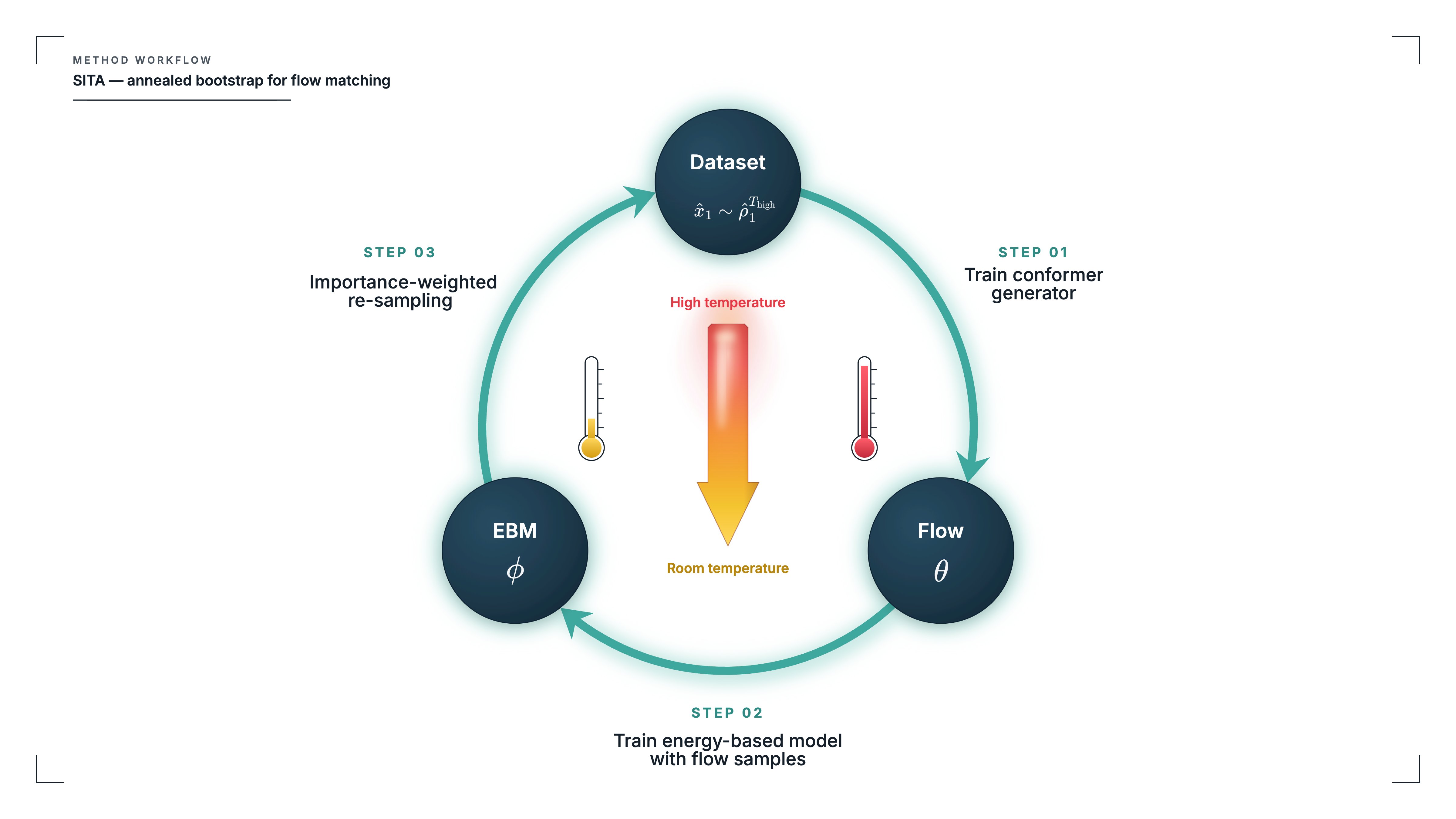}
\caption{SITA training loop: A flow model
$\theta$ trained on high-temperature samples is used to generate proposals for training an energy-based model
$\phi$. Importance-weighted resampling with the learned surrogate likelihoods produces samples at lower temperatures, which seeds the next annealing step without expensive Jacobian computations.
}
\label{fig:sita-framework}
\end{figure}


\paragraph{Main Contributions.} In this work, we present SITA (Figure~\ref{fig:sita-framework}), the first inference-time annealing for continuous flow-matching (\citet{lipman2022flow}) models that circumvents path integral-based SNIS estimators via reliance on a new class of surrogate likelihood estimators called \textit{BoltzNCE} (\citet{aggarwal2025boltznce}). On the benchmark systems of alanine dipeptide and alanine tripeptide, we demonstrate the following:

\begin{itemize}

    \item \textbf{Annealing for flows.} We develop a simple inference-time annealing procedure for continuous flow models that allows large temperature jumps across a pre-defined temperature ladder.

    \item  \textbf{Surrogate-driven annealed importance sampling.} We integrate a BoltzNCE-style surrogate into a temperature annealing ladder, enabling importance-weighted transport from the high-temperature Boltzmann distribution to the room temperature equilibrium target distribution — a regime where direct Jacobian-based reweighting would be prohibitive.

    \item \textbf{Empirical validation.} SITA achieves state-of-the-art results across several metrics despite the inherent bias introduced by the surrogate likelihood estimator.

\end{itemize}


\section{Background}

To address the prohibitive cost of likelihood evaluation in flow-based Boltzmann samplers, SITA combines a flow-based generative model with a surrogate likelihood estimator to enable efficient inference-time annealing for molecular Boltzmann sampling. We first review stochastic interpolants, a unifying framework for continuous-time generative modeling that subsumes diffusion models \& flow matching, which provides the generative backbone of our approach. We then describe BoltzNCE, a recently developed energy-based surrogate likelihood estimator that yields the tractable density evaluations our reweighting procedure relies on.

\subsection{Stochastic Interpolants and Flow Matching}

A stochastic interpolant \citet{albergo2023stochastic,ma2024sit} defines a continuous-time process that transports samples from a base distribution $\rho_0$ to a target distribution $\rho_1 = \pi$. For paired samples $x_0 \sim \rho_0$ and $x_1 \sim \rho_1$, the interpolant takes the form
\begin{equation*}
	\label{interpolant}
	I_t = \alpha_t x_0 + \beta_t x_1,
\end{equation*}
where $\alpha, \beta : [0,1] \to \mathbb{R}$ are continuously differentiable scalar functions satisfying $\alpha_0 = 1$, $\alpha_1 = 0$, $\beta_0 = 0$, and $\beta_1 = 1$. This formulation encompasses diffusion models \citet{song2020score,ho2020denoising}, flow matching \citet{lipman2022flow}, and rectified flows \citet{liu2022flow} as special cases under appropriate choices of $\alpha$ and $\beta$.

The time-marginal density $\rho_t = \text{Law}(I_t)$ coincides with the density evolved by a probability flow transporting mass from $\rho_0$ to $\rho_1$:
\begin{equation*}
	\label{si_ode}
	\dot{X}_t = v_t(X_t), \qquad v_t(x) = \mathbb{E}[\dot{I}_t \mid I_t = x],
\end{equation*}
where the velocity field $v_t$ is characterized as the conditional expectation of the interpolant's time derivative.

Given a coupling $\rho(x_0, x_1)$ between the base and target distributions, the velocity field can be learned by training a neural network $\hat{v}$ via the regression loss
\begin{equation}
	\label{vector field objective SI}
	\mathcal{L}_{v}(\hat{v}) = \mathbb{E}\left[ \|\hat{v}_t(I_t) - \dot{I}_t \|^2\right],
\end{equation}
where the expectation is taken over $(t, x_0, x_1)$. Samples are then generated by integrating the learned dynamics
\begin{equation*}
    \label{flow matching ODE}
    \dot{x}_t = \hat{v}_t(x_t), \quad x_0 \sim \rho_0.
\end{equation*}
The density of the generated samples is given by the change-of-variables formula
\begin{equation*}
    \label{continuity equation}
    \hat{\rho}_{1}(\hat{x}_{1}) = \rho_{0}(x_{0})\exp\left(-\int^{1}_{0} \nabla \cdot v_{t}(x_{t})\, dt\right).
\end{equation*}

\subsection{Surrogate Likelihood Estimators (BoltzNCE)}

SITA employs a surrogate likelihood estimator parameterized as an energy-based model (EBM) to enable efficient reweighting of samples produced by the flow. Training EBMs was historically considered intractable, but recent advances, particularly BoltzNCE (\cite{aggarwal2025boltznce}), have made it practical by combining score matching with noise contrastive estimation (NCE) (\cite{gutmann2010noise,oord2018representation}).

BoltzNCE training proceeds in two stages. First, samples $\hat{x}_1 \sim \hat{\rho}_1$ are drawn from the flow to serve as training data for the EBM. Second, the EBM is fitted to these samples to learn an energy function $U$ over the flow's output distribution $\hat{\rho}_1$.

The training objective is derived by introducing an auxiliary stochastic interpolant from a Gaussian base to $\hat{\rho}_1$. The score matching component takes the form
\begin{equation*}
    \label{bnc_sm}
    \mathcal{L}_{\text{SM}}(\hat{U}) = \mathbb{E}\left[|\alpha_t \nabla\hat{U}_t(\tilde{I}_t) + x_0|^2\right],
\end{equation*}
where the expectation is over $(t, x_0, \hat{x}_1)$ with $\tilde{I}_t = \alpha_t x_0 + \beta_t \hat{x}_1$. While score matching constrains the gradient of the energy, an additional NCE-based term is needed to anchor the energy values themselves. This is achieved by training the model to discriminate between samples from different time points:
\begin{equation*}
	\label{InfoNCE}
	\mathcal{L}_{\text{InfoNCE}}(\hat{U}) = -\mathbb{E}\left[\log \left(\frac{\exp(\hat{U}_t(\tilde{I}_t))}{\exp(\hat{U}_{t^{\prime}}(\tilde{I}_t)) + \exp(\hat{U}_{t}(\tilde{I}_t))}\right)\right],
\end{equation*}
where $t'$ denotes a contrastive time point and the expectation is over $(t', t, x_0, \hat{x}_1)$. The full BoltzNCE objective combines both terms:
\begin{equation}
	\label{eqn:combined_loss}
	\mathcal{L}_{\text{BoltzNCE}}(\hat{U}) = \mathcal{L}_{\text{SM}}(\hat{U}) + \mathcal{L}_{\text{InfoNCE}}(\hat{U}).
\end{equation}

\section{Scalable Inference-Time Annealing}

In this section, we outline how through a combination of surrogate likelihood estimation and flow annealing we facilitate a highly scalable algorithm to model the Boltzmann distribution of physical systems. Let $\{T_k\}^{K}_{k=0}$ denote a decreasing sequence of temperatures and $E(\cdot)$ be the energy function associated with the target Boltzmann distribution $\pi(x_1)$, which we can evaluate exactly. Given a flow with parameters $\theta$ pre-trained on high-temperature simulation data and an EBM with parameters $\phi$ pre-trained on flow generated outputs, SITA's multi-phase annealing bootstrap proceeds as follows:

\begin{enumerate}[label=(\roman*)]
    \item \textbf{Anneal the flow.} Samples $\hat{x}_{1} \sim \hat{\rho}^{T_{k+1}}_{1}$ are generated from the flow by drawing $x_{0} \sim \mathcal{N}(0, \frac{T_{k+1}}{T_{k}}\mathbf{I})$ and integrating the generative ODE.
    \item \textbf{Finetune the EBM.} Using flow generated samples, the EBM is finetuned to better approximate the likelihoods of $\hat{x}_{1} \sim \hat{\rho}^{T_{k+1}}_{1}$. 
    \item \textbf{Importance sampling.} We use the newly fitted EBM to compute importance weights $\tilde{w}(\hat{x}_1) = \exp(-\frac{1}{k_BT_{k+1}}E(\hat{x}_1) - \hat{U}_{\phi}(\hat{x}_1))$ for each sample generated by the flow. Subsequently, samples are re-weighted to provide a new dataset at temperature $T_{k+1}$ to retrain the flow.
    \item \textbf{Finetune the flow.} Using the set of importance weighted samples, the flow is finetuned to better approximate the annealed target distribution. Upon completion, the whole process begins again until the final temperature $T_{K}$ is reached.
\end{enumerate}

A workflow diagram of SITA is illustrated in Figure~\ref{fig:sita-framework}, accompanied with pseudocode in Algorithm \ref{alg:sita}. We emphasize that the same two models are maintained throughout the entire annealing bootstrap process. Their respective optimizers are simply re-initialized at the start of each new finetuning step.

\begin{algorithm}[H]
\caption{Scalable Inference-Time Annealing (SITA)}
\label{alg:sita}
\begin{algorithmic}[1]
\REQUIRE Temperature ladder $\{T_k\}_{k=0}^{K}$ (decreasing), energy function $E(\cdot)$ defining the target distribution $\pi$
\REQUIRE Pre-trained flow $\hat{v}_\theta$, pre-trained EBM $\hat{U}_\phi$
\ENSURE Flow model $\hat{v}_\theta$ approximates $\pi$ at the target temperature $T_K$
\FOR{$k = 0$ to $K-1$}
    \STATE \textit{// Anneal the flow}
    \STATE Draw $x_0 \sim \mathcal{N}(0, \frac{T_{k+1}}{T_k}\mathbf{I})$
    \STATE Generate $\hat{x}_1 \sim \hat{\rho}_1^{T_{k+1}}$ by integrating generative ODE
    \STATE
    \STATE \textit{// Finetune the EBM}
    \STATE Reinitialize optimizer for $\phi$
    \STATE Update $\phi$ using $\mathcal{L}_{\text{BoltzNCE}}(\hat{U}_\phi)$ (see Equation \ref{eqn:combined_loss}) for a total of $N_{\text{EBM}}$ epochs
    \STATE
    \STATE \textit{// Importance-weighted resampling}
    \STATE Compute log weights $\log\tilde{w}(\hat{x}_1) \leftarrow -\frac{1}{k_B T_{k+1}}E(\hat{x}_1) - \hat{U}_\phi(\hat{x}_1)$
    \STATE
    \STATE \textit{// Clip at 99th percentile}
    \STATE $\log\tilde{w}(\hat{x}_1) \leftarrow \texttt{quantile\_clip}(\log\tilde{w}(\hat{x}_1), 0.99)$
    \STATE Resample $\{\hat{x}_1\}$ according to self-normalized weights $w(\hat{x}_1)$
    \STATE
    \STATE \textit{// Finetune the flow}
    \STATE Reinitialize optimizer for $\theta$
    \STATE Update $\theta$ on resampled data using $\mathcal{L}_{v}(\hat{v}_\theta)$ (see Equation \ref{vector field objective SI}) for a total of $N_{\text{Flow}}$ epochs
\ENDFOR
\RETURN $\hat{v}_\theta$
\end{algorithmic}
\end{algorithm}

\subsection{Temperature Steering of the Probability Flow ODE}

We observe that velocity field models can induce temperature changes in the generated distribution without explicit temperature conditioning during training. A velocity field $\hat{v}$ trained on simulation data at temperature $T_{\text{high}}$ transports samples from the base distribution $\rho_0(x_0) = \mathcal{N}(0, \mathbf{I})$ to produce $\hat{x}_1 \sim \hat{\rho}^{T_{\text{high}}}_1$, where
\begin{equation*}
    \hat{\rho}^{T_{\text{high}}}_1(\hat{x}_1) = \rho_0(x_0) \exp\left(-\int_0^1 \nabla \cdot \hat{v}_t(x_t)\, dt\right).
\end{equation*}
Crucially, the flow learns to map $\rho_0$ to an approximation of the high-temperature target without explicitly modeling the temperature dependence. This information is encoded implicitly in the flow's output distribution.

This implicit encoding enables a simple mechanism for temperature transfer. By raising the flow density to the power $\kappa = T_{\text{high}} / T_{\text{low}}$, we obtain the proportionality relation
\begin{equation*}
    \hat{\rho}^{T_{\text{low}}}_1(\hat{x}_1) \propto \left[\hat{\rho}^{T_{\text{high}}}_1(\hat{x}_1)\right]^\kappa = \left[\rho_0(x_0) \exp\left(-\int_0^1 \nabla \cdot \hat{v}_t(x_t)\, dt\right)\right]^\kappa.
\end{equation*}
Consequently, sampling an approximation of the low-temperature distribution $\hat{\rho}^{T_{\text{low}}}_1$ simply requires a rescaling of the base distribution variance by $\kappa^{-1}$ at inference time:
\begin{equation*}
    \left[\rho_0(x)\right]^\kappa = \exp\left(-\frac{T_{\text{high}} \|x\|^2}{2 T_{\text{low}}}\right) \propto \mathcal{N}(0, \kappa^{-1} \mathbf{I}).
\end{equation*}
This rescaling allows SITA to achieve large temperature jumps without modifying the velocity field model's architecture or enforcing volume preservation under the instantaneous change of variables. Further details on temperature steerability are provided in Appendix~\ref{app:tsf}.

\subsection{Importance Sampling with Surrogate Likelihood Estimators}

The reliance on an EBM as a surrogate for exact likelihoods of generated samples naturally introduces bias into the estimates of importance weights. Concretely, for any test function $f(\cdot)$, generator $\rho_{\theta}$, and density model $q_{\phi}$, the self-normalized estimator converges to:

\begin{equation*}
    \frac{ \mathbb{E}_{\rho_{\theta}} \left[ w(x) f(x) \right] }{\mathbb{E}_{\rho_{\theta}} \left[ w(x) \right]} = \frac{ \int \rho_{\theta}(x) \frac{\pi(x)}{q_{\phi}(x)} f(x) dx }{ \int \rho_{\theta}(x) \frac{\pi(x)}{q_{\phi}(x)} dx } = \frac{ \int \frac{\rho_{\theta}(x)}{q_{\phi}(x)} \pi(x) f(x) dx }{ \int \frac{\rho_{\theta}(x)}{q_{\phi}(x)} \pi(x) dx }
\end{equation*}

Clearly, the recovered distribution is a tilted version of the target

\begin{equation}
    \label{equ:bias}
    \tilde{\pi}(x) \propto \frac{\rho_{\theta}(x)}{q_{\phi}(x)} \pi(x)
\end{equation}
 where the target $\pi$ is recovered exactly only when $\rho_{\theta} = q_{\phi}$. Empirically, we demonstrate that despite the bias introduced by SITA's surrogate likelihood estimator, the EBM derived importance weights deliver superior performance on benchmark molecular systems.

\section{Related Work}

Section \ref{sec:intro} situates our contribution within the broader literature on bootstrap generative modeling, where temperature annealing for continuous-time flows remains conspicuously under-explored --- a gap this work directly addresses.

\paragraph{Flow-based Boltzmann samplers.} Normalizing flows were first applied to molecular sampling by (\citet{noe2019boltzmann}), who introduced Boltzmann Generators as bijections from a tractable base to a learned approximation of the Boltzmann distribution. Subsequent work improved expressivity via equivariant architectures (\citet{kohler2020equivariant, satorras2021en}) and via classical force-field couplings (\citet{wirnsberger2020targeted}), but mode coverage has remained the binding empirical constraint: a maximum-likelihood objective evaluated on equilibrium data cannot, on its own, place mass on modes that are absent from the training set. SITA retains the architectural advantages of continuous-time flows while replacing the maximum-likelihood objective with a temperature-annealing bootstrap that progressively transfers samples from a tractable high-temperature distribution to the cold target, sidestepping the circularity that constrains direct equilibrium training.

\paragraph{Diffusion-based Sampling.} Amortized samplers built on diffusion models have proliferated rapidly, with methods distinguished primarily by how the optimal drift is estimated and the cost paid to obtain it. Simulation-based approaches (\citet{zhang2022path,berner2024improved,vargas2023denoising,ctmd}) train a learned diffusion against the Boltzmann log-density, exploiting the fast mode mixing of the forward process at the cost of full trajectory simulation during training. Conversely, simulation-free alternatives such as iDEM (\citet{idem}) and TSM (\citet{deBortoli2024target}) sidestep this simulation cost by regressing on the score directly, achieving better scaling at the price of high-variance estimates of the score function. More recent work explores alternative bridge constructions, including Schrödinger bridges (\citet{asbs, bridge_matching}) and underdamped Langevin bridges (\citet{blessing2024underdamped}). In parallel, iterative bootstrap procedures (\citet{akhound2025progressive, ptsd}) have been introduced to address the mode-coverage gap by progressively annealing the sampler toward the target temperature.

\paragraph{Bootstrapping via Importance Sampling.} A complementary line of work bootstraps a flow against itself by exploiting an importance-sampling correction along an annealing path, drawing on Annealed Importance Sampling (\citet{neal}) and Sequential Monte Carlo (\citet{smc}). FAB (\citet{fab}) uses AIS to construct importance-weighted training targets along a geometric path; (\citet{cmt_bless}) extends this with constrained mass transport; TA-BG (\citet{temp_bg}) replaces the geometric path with a temperature schedule. 

\paragraph{Continuous-time Generative Models.} SITA's parameterization builds on continuous-time generative models (\citet{chen2018neuralode, song2021score}) and the stochastic interpolant framework (\citet{albergo2023stochastic,ma2024sit}), which subsumes flow matching (\citet{lipman2023flow}) and rectified flows (\citet{liu2023rectified}) as special cases. These models avoid the architectural restrictions of coupling-based flows but trade exact likelihood for ODE-based likelihood evaluation, whose cost scales prohibitively with the dimension of the state space and with the desired numerical accuracy of the ODE solver.

\section{Experiments}
\label{sec:experiments}

We evaluate SITA against PITA on two molecular benchmark systems: alanine dipeptide and alanine tripeptide. For our bootstrapping experiments, we pre-train SITA's flow model on the same 1200K MD simulation data as PITA. All evaluation metrics reported in this section are calculated against the same test-split used in PITA, containing randomized frames from a 300K MD simulation. All metrics were evaluated across 3 seeds by generating 10,000 samples from SITA's flow post-bootstrap. To quantify mode coverage and precision, we compute Wasserstein distances over Ramachandran coordinates ($\mathbb{T}$-$\mathbb{W}_{2}$) and energy distributions ($\mathcal{E}$-$\mathbb{W}_{1}$, $\mathcal{E}$-$\mathbb{W}_{2}$). Additionally, we measure KL divergence between ground-truth and generated Ramachandran distributions (Rama-KL). Lastly, we account for all MD energy evaluations consumed by both methods. After the shared upfront cost of generating the training set ($5 \times 10^{7}$ evaluations for each system), bootstrapping incurs one to two orders of magnitude fewer evaluations, reflecting the principal efficiency claim of SITA.

\paragraph{Baselines} Beyond PITA, we include three baselines from \citet{akhound2025progressive}: Temperature Annealed Boltzmann Generators (TA-BG; \citealt{temp_bg}), a diffusion model (MD-Diff), and a normalizing flow (MD-NF). MD-Diff and MD-NF are trained directly on MD samples simulated at 300K, while TA-BG is a normalizing flow trained via temperature-annealed bootstrapping. Energy-evaluation reporting for MD-Diff and MD-NF reflects the cost of generating their MD training data, while reporting for TA-BG reflects the cost of bootstrapping only.

\paragraph{Architecture.} SITA's flow model is a Geometric Vector Perceptron graph neural network, introduced by \cite{GVP}, equipped with $E(3)$-invariance in its scalar features and E(3)-equivariance in its vector features. To parameterize the EBM, we use the Graphormer architecture \cite{ying2021transformers} due to its successful application in \cite{aggarwal2025boltznce}. Both flow and score matching objectives used for training make use of trigonometric interpolants introduced in \cite{albergo2023stochastic}, where $\alpha_{t} = \cos(\frac{\pi}{2}t)$ and $\beta_{t} = \sin(\frac{\pi}{2}t)$.

\subsection{Main results}

\begin{figure}[H]  
\centering
\includegraphics[width=\textwidth]{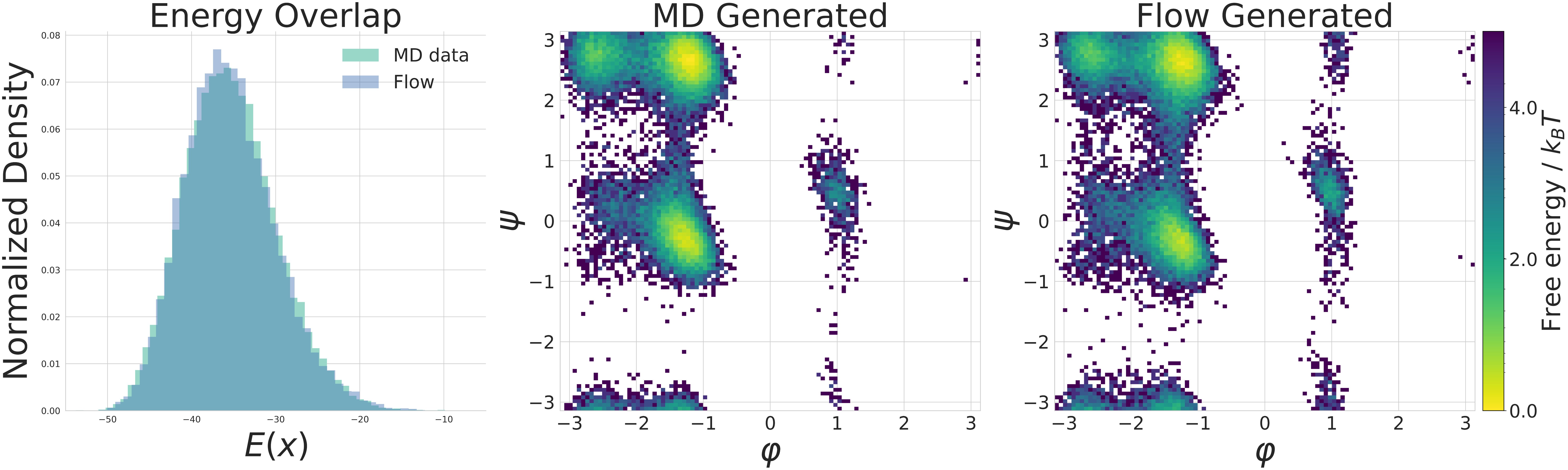}
\caption{Alanine dipeptide comparison on 30,000 samples from both SITA and MD simulation. Flow-generated samples closely match the MD reference in both energy distribution and Ramachandran free energy landscape, capturing all major conformational basins.
}
\label{fig:energy_rama}
\end{figure}
\vspace{-1.75em}
\paragraph{Alanine Dipeptide (ADP).} SITA is applied to the task of sampling conformations of alanine dipeptide at the target temperature of 300K, having been initially pre-trained on an MD dataset simulated at 1200K. We define an annealing schedule over temperatures 755.95K, 555.52K, 408.24K, 300.00K. We note this includes one extra step of 408.24K compared to PITA applied to alanine dipeptide.

At each annealing step, SITA's annealed flow generates 200,000 samples to fine-tune the EBM. A new set of 100,000 samples is then produced by importance weighted re-sampling of the original 200,000 samples using the EBM to calculate the SNIS estimate. We report the effective sample size (ESS) for each annealing step in Table~\ref{adp-ess-table}. SITA achieves the best performance on Rama-KL and the energy Wasserstein-2 metric, as can be seen from Table~\ref{adp-results-table}. While SITA outperforms PITA on the energy Wasserstein-1 metric, it is itself outperformed by MD-NF, a baseline trained directly on equilibrium samples from MD at 300K. The poor Rama-KL of MD-NF, however, suggests this conclusion is misleading: MD-NF likely produces low-energy conformers from only a subset of the modes, a signature of mode collapse. Visual comparison confirms that SITA captures all major conformational basins and matches the reference energy distribution (Figure~\ref{fig:energy_rama}). SITA remains competitive with PITA on the torsion metric ($\mathbb{T}$-$\mathcal{W}_2$), where accurately recovering mode populations remains challenging.

\vspace{10pt}
\begingroup

\begin{table}[H]
\centering
\resizebox{\textwidth}{!}{%
\begin{tabular}{lcccccc}
\toprule
 & Rama-KL $\downarrow$& Energy-$\mathcal{W}_1$ $\downarrow$ & Energy-$\mathcal{W}_2$ $\downarrow$ & $\mathbb{T}$-$\mathcal{W}_2$ $\downarrow$& \#Energy Evals $\downarrow$\\
\midrule
PITA & $4.773 \pm 0.460$ & $1.530 \pm 0.068$ & $1.615 \pm 0.053$ & $\mathbf{0.270 \pm 0.023}$ & $5.01 \times 10^7$ \\
TA-BG & $14.993 \pm 0.002$ & $83.457 \pm 0.070$ & $86.176 \pm 0.104$ & $0.979 \pm 0.012$ & $5 \times 10^{7}$ \\
MD-Diff & $1.308 \pm 0.072$ & $3.627 \pm 0.023$ & $3.704 \pm 0.026$ & $0.310 \pm 0.001$ & $5 \times 10^7$ \\
MD-NF & $13.533 \pm 0.024$ & $\mathbf{0.551 \pm 0.062}$ & $1.198 \pm 0.069$ & $0.403 \pm 0.045$ & $5 \times 10^7$ \\
SITA & $\mathbf{0.517 \pm 0.013}$ & $0.865 \pm 0.080$ & $\mathbf{0.939 \pm 0.079}$ & $0.326 \pm 0.004$ & $5.08 \times 10^7$ \\
\bottomrule
\end{tabular}%
}
\caption{Alanine Dipeptide at 300K. Metrics calculated over 10,000 samples across 3 seeds.}
\label{adp-results-table}
\end{table}

\begin{table}[H]
\centering
\begin{tabular}{lcccc}
\toprule
 & 755.95K & 555.52K & 408.24K & 300.00K \\
\midrule
ESS $\uparrow$ & 0.131 & 0.218 & 0.192 & 0.268 \\
\bottomrule
\end{tabular}
\caption{ESS across temperatures for ADP.}
\label{adp-ess-table}
\end{table}
\endgroup

\paragraph{Alanine Tripeptide (ATP).} We further evaluate SITA on the larger alanine tripeptide system, using the same annealing schedule and bootstrapping procedure as in the alanine dipeptide experiment. Note that in this application, the annealing schedules between SITA and PITA are now the same. Here, we observe that SITA outperforms all methods on nearly all metrics (Table \ref{atp-results-table}) with the exception of the torsion metric ($\mathbb{T}$-$\mathcal{W}_2$). Remarkably, SITA is able to achieve this superior performance without requiring any relaxation of generated samples. In contrast, PITA and TA-BG must run a short MD refinement on their generated samples at the target temperature to remain competitive.

\vspace{10pt}
\begingroup

\begin{table}[H]
\centering
\resizebox{\textwidth}{!}{%
\begin{tabular}{lccccr}
\toprule
 & Rama-KL & Energy-$\mathcal{W}_1$ $\downarrow$ & Energy-$\mathcal{W}_2$ $\downarrow$ & $\mathbb{T}$-$\mathcal{W}_2$ & \#Energy Evals \\
\midrule
PITA                  & $1.209 \pm 0.144$ & $2.567 \pm 0.108$ & $2.592 \pm 0.107$ & $0.521 \pm 0.006$ & $8 \times 10^7$ \\
PITA (w/o relaxation) & $8.535 \pm 0.254$          & $86.270 \pm 0.294$         & $87.695 \pm 0.294$         & $0.651 \pm 0.013$ & $5.01 \times 10^7$ \\
TA-BG & $2.078 \pm 2.088$ & $4.782 \pm 0.076$ & $4.863 \pm 0.082$ & $\mathbf{0.347 \pm 0.014}$ & $8 \times 10^7$\\
TA-BG (w/o relaxation) & $14.988 \pm 0.009$ & $173.042 \pm 0.717$ & $178.558 \pm 0.732$ & $1.310 \pm 0.004$ & $8 \times 10^7$ \\
MD-Diff               & $9.662 \pm 0.085$          & $7.416 \pm 0.130$          & $7.599 \pm 0.137$          & $0.424 \pm 0.011$ & $8 \times 10^7$ \\
SITA (w/o relaxation) & $\mathbf{0.361 \pm 0.025}$ & $\mathbf{1.933 \pm 0.298}$ & $\mathbf{2.054 \pm 0.268}$ & $0.798 \pm 0.268$ & $5.08 \times 10^7$ \\
\bottomrule
\end{tabular}
}
\caption{Alanine Tripeptide at 300K. Metrics calculated over 10,000 samples across 3 seeds.}
\label{atp-results-table}
\end{table}

\begin{table}[H]
\centering
\begin{tabular}{lcccc}
\toprule
 & 755.95K & 555.52K & 408.24K & 300.00K \\
\midrule
ESS $\uparrow$ & 0.045 & 0.067 & 0.074 & 0.065 \\
\bottomrule
\end{tabular}
\caption{ESS across temperatures for ATP.}
\label{atp-ess-table}
\end{table}

\endgroup

\subsection{Metropolis Hastings Refinement}

Lastly, we apply the Independent Metropolis-Hastings (IMH) algorithm \citet{amc_nf} as a post-hoc refinement of generated samples, using both SITA's flow and EBM in the accept-reject step. IMH is a Markov chain Monte Carlo (MCMC) method whose accept-reject mechanism guarantees convergence to a target distribution, provided the densities of both the target and the proposal can be evaluated. Our goal in this setting is to obtain samples whose empirical distribution more closely approximates the target Boltzmann distribution $\pi(x)$ at $300$K. In the absence of an exact proposal density, we replace the proposal in the IMH acceptance probability with a learned surrogate $q_\phi$, yielding
\begin{equation*}
    \alpha(x, y) = \min \left( 1, \frac{\pi(y)\, q_\phi(x)}{\pi(x)\, q_\phi(y)} \right).
\end{equation*}

Inclusion of the EBM introduces a bias that breaks the guarantee of $\pi(x)$ as the stationary distribution of the chain; the chain instead converges to the tilted distribution $\tilde{\pi}(x)$, defined in Equation~\ref{equ:bias} (see Appendix \ref{app:IMH} for full derivation). Despite this bias, we demonstrate on alanine tripeptide that the IMH refinement still improves SITA's performance. We compare IMH refinement, ran for 50 steps, against unadjusted SITA and SITA augmented with an additional importance sampling step (SITA-IS). Each experiment uses 10,000 samples and is repeated across three random seeds. Results appear in Table~\ref{imh-results-table}, with energy evaluation counts reflecting the post-bootstrap phase only. Further algorithmic details are available in Appendix~\ref{app:code}. 

\begin{table}[H]
\centering

\resizebox{\textwidth}{!}{%
\begin{tabular}{lcccccr}
\toprule
 & ESS $\uparrow$ & Rama-KL & Energy-$\mathcal{W}_1$ $\downarrow$ & Energy-$\mathcal{W}_2$ $\downarrow$ & $\mathbb{T}$-$\mathcal{W}_2$ & \#Energy Evals \\
\midrule
SITA & --- & $0.361 \pm 0.025$ & $1.933 \pm 0.298$ & $2.054 \pm 0.268$ & $0.798 \pm 0.268$ & --- \\
SITA-IS & $\mathbf{0.191 \pm 0.006}$ & $0.991 \pm 0.042$ & $\mathbf{0.681 \pm 0.310}$ & $\mathbf{0.734 \pm 0.287}$ & $0.770 \pm 0.004$ & $10^4$ \\
SITA-IMH (50 steps) & --- & $\mathbf{0.313 \pm 0.020}$ & $0.996 \pm 0.029$ & $1.702 \pm 0.025$ & $\mathbf{0.704 \pm 0.029}$ & $5 \times 10^5$ \\
\bottomrule
\end{tabular}
}
\caption{Independent Metropolis Hastings refinement for Alanine Tripeptide at 300K. All metrics calculated over 10,000 samples across 3 seeds.}
\label{imh-results-table}
\end{table}

From Table~\ref{imh-results-table}, we observe improvements over unadjusted SITA across the reported metrics. SITA-IS achieves the lowest Energy-$\mathcal{W}_1$ and Energy-$\mathcal{W}_2$ values; however, it incurs a marked degradation in Rama-KL and yields only a marginal improvement in $\mathbb{T}$-$\mathcal{W}_2$. This trade-off reflects a loss of sample diversity introduced by importance-weighted resampling, as evidenced by the modest effective sample size of $0.191$. SITA-IMH improves every metric while preserving sample diversity, attaining the best Rama-KL and $\mathbb{T}$-$\mathcal{W}_2$ scores, albeit at the cost of additional energy evaluations.

\subsection{TICA Evaluations}

We evaluate generative model performance using time-lagged independent component analysis (TICA) \cite{noe_tica, tica_pande}, reporting Wasserstein-1 and Wasserstein-2 distances between TICA projections of generated and MD samples. We find that two methodological choices — the lag time and the protocol for down-sampling MD trajectory frames — meaningfully affect the resulting metrics, and we outline below the protocol we adopt. A fuller discussion, including comparison with the values reported in \citet{akhound2025progressive}, is provided in Appendix~\ref{app:tica}.

\begin{figure}[!htbp]
\centering
\includegraphics[width=\textwidth]{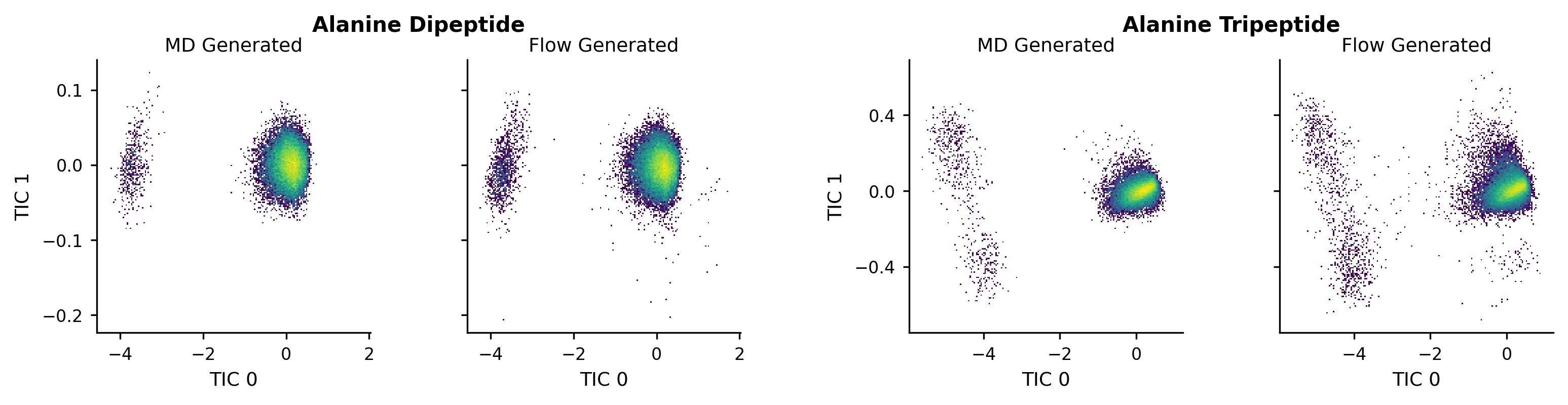}
\caption{TICA projection density scatter plots comparing MD-generated and SITA flow-generated samples for alanine dipeptide (left) and alanine tripeptide (right). Each panel plots TICA projections of 30,000 samples, with color intensity reflecting local sample density.}
\label{fig:tica_projections}
\end{figure}

MD trajectories are not i.i.d.\ samples and exhibit strong time correlations. Down-sampling frames for TICA projection by selecting a contiguous block of initial frames can drop modes or distort their relative weights. We therefore fit TICA on the full MD trajectory and uniformly subsample frames for the projection step. We also advocate a single standardized lag time across systems: small lag times render TICA sensitive to fast motions that may be indistinguishable from noise, and per-system tuning complicates cross-system comparison. To apply these conventions consistently in our comparison with \citet{akhound2025progressive}, we regenerated MD trajectories for both alanine dipeptide and alanine tripeptide following the simulation configuration described therein, as the original trajectories were not publicly available.

Table~\ref{tica-table} reports TICA-$\mathcal{W}_{1}$ and TICA-$\mathcal{W}_{2}$ metrics under both down-sampling protocols, with significant improvement under uniform sampling. For direct comparability with the values in \citet{akhound2025progressive}, we use lag times of 100 for ADP and 10 for ATP; recomputed results for ATP with a lag time of 100 show substantially closer agreement with MD and are reported in Appendix~\ref{app:tica}. Figure~\ref{fig:tica_projections} shows density scatter plots of the first two TICA components for both systems, illustrating qualitative agreement between SITA and MD.

\begin{table}[H]
\centering
\resizebox{\textwidth}{!}{%
\begin{tabular}{lcccccccc}
\toprule
 & \multicolumn{4}{c}{ADP TICA} & \multicolumn{4}{c}{ATP TICA} \\
\cmidrule(lr){2-5} \cmidrule(lr){6-9}
 & $\mathcal{W}_1$ & $\mathcal{W}_2$ & $\mathcal{W}_1^*$ & $\mathcal{W}_2^*$ 
 & $\mathcal{W}_1$ & $\mathcal{W}_2$ & $\mathcal{W}_1^*$ & $\mathcal{W}_2^*$ \\
\midrule
PITA                  & ---               & ---               & ---               & ---               & $0.272 \pm 0.017$ & $0.952 \pm 0.055$ & --- & --- \\
PITA (w/o relax.)     & $0.118 \pm 0.006$ & $0.379 \pm 0.028$ & ---               & ---               & $0.405 \pm 0.014$ & $0.999 \pm 0.043$ & --- & --- \\

TA-BG                  & ---               & ---               & ---               & ---               & $0.082 \pm 0.001$ & $0.454 \pm 0.001$ & --- & --- \\
TA-BG (w/o relax.)     & $0.219 \pm 0.013$ & $0.685 \pm 0.034$ & ---               & ---               & $0.321 \pm 0.001$ & $0.648 \pm 0.000$ & --- & --- \\

MD-Diff               & $0.113 \pm 0.001$ & $0.579 \pm 0.004$ & ---               & ---               & $0.059 \pm 0.006$ & $0.426 \pm 0.010$ & --- & --- \\
MD-NF                 & $0.138 \pm 0.003$ & $0.586 \pm 0.003$ & ---               & ---               & --- & --- & --- & --- \\
SITA (w/o relax.)                  & $0.155 \pm 0.009$ & $0.629 \pm 0.212$ & $0.089 \pm 0.012$ & $0.414 \pm 0.035$ & $0.311 \pm 0.014$ & $0.943 \pm 0.021$ & $0.179 \pm 0.012$ & $0.546 \pm 0.035$ \\
\bottomrule
\end{tabular}
}
\caption{TICA-$\mathcal{W}_1$ and TICA-$\mathcal{W}_2$ metrics for alanine dipeptide (ADP) and alanine tripeptide (ATP). Columns marked with $^*$ denote metrics computed using uniform frame sampling along the MD trajectory, while unmarked columns use the first-frame down-sampling method from the original PITA paper. All metrics are calculated over 10,000 samples across 3 seeds.}
\label{tica-table}
\end{table}

\section{Conclusion}
\label{sec:conclusion}

We introduce SITA, a scalable framework that trains continuous flow models to recover target Boltzmann distributions via temperature annealing, achieving computational efficiency by avoiding vector-field divergence computations. Our method establishes a new state-of-the-art on both alanine dipeptide and alanine tripeptide. We further demonstrate empirically that surrogate likelihood estimators offer a tractable route to modeling molecular ensembles with many degrees of freedom, a regime in which existing methods struggle. Future directions include architectural optimization, cross-system transferability, and applications to larger molecular systems.

\acks{This research was undertaken thanks to funding through R35GM140753 from the National Institute of General Medical Sciences and the American Chemical Society's Graduate Student Success Grant}

\newpage
\appendix

\section{Temperature Steerability in Flows}
\label{app:tsf}

\subsection{Normalizing Flows} In their original formulation, normalizing flows learn a differentiable bijection $x = \phi_{\theta}(z)$ that maps samples from a simple base distribution $\rho_{Z}(z)$, typically a multivariate Gaussian, to some target distribution $\pi(x)$. By the change-of-variables formula, likelihoods under the normalizing flow can be evaluated as 

\begin{equation}
    \label{norm_flow}
    \rho_{\theta}(x) = \rho_{Z}(\phi^{-1}_{\theta}(x)) \left| \frac{\partial\phi^{-1}_{\theta}}{\partial x}(x) \right|,
\end{equation}

where $z = \phi^{-1}_{\theta}(x)$ and $\left| \partial \phi^{-1}_{\theta}/\partial x \right|$ is the Jacobian determinant of the flow map $\phi_{\theta}$, which is comprised of a discrete sequence of invertible neural network layers with parameters $\theta$. In the continuous-time perspective, a continuous-time normalizing flows (CNF) replaces this stack of maps with a time-indexed family of maps $X_{t}(x) = \phi_{t}(x_{0})$ with $t \in [0, 1]$ that acts as the pushforward of the base distribution $\rho_{0}$ at time $t=0$ to some time-varying density $\rho_t$ that terminates at $\rho_{1}$ at time $t=1$. Conveniently, the flow map $\phi_{t}$ is characterized by the ordinary differential equation (ODE)

\begin{equation}
    \Dot{X}_{t}(x) = \frac{d}{dt}\phi_{t}(x) = v_{t}(\phi_{t}(x)), \qquad X_{t=0}(x) = x
\end{equation}

where $v_{t}$ is the \textit{velocity field} governing the transport of individual samples. From the density perspective, the pushforward of $\rho_{0}$ under $\phi_{t}$ yields a time-varying density $\rho_{t}$ that satisfies the continuity equation

\begin{equation}
    \label{ceq}
    \partial_{t} \rho_{t} + \nabla \cdot (v_{t}\rho_{t})= 0 \qquad \text{with boundaries } \rho_{t=0} = \rho_{0} \text{ and } \rho_{t=1} = \rho_{1}.
\end{equation}

Denoting $\hat{\rho}_{1}$ as the CNF's output distribution,  likelihoods for generated samples $\hat{x}_{1}$ are given by the instantaneous change-of-variables formula,

\begin{equation}
    \hat{\rho}_{1}(\hat{x}_{1}) = \rho_{0}(x_{0})\exp\left(-\int^{1}_{0} \nabla\cdot v_{t}(x_{t}) dt\right).
\end{equation}

Generative modeling thus reduces to estimating a velocity field $v_{t}$ such that \cref{ceq} is satisfied.

\subsection{Explicit Temperature Steering} 
 As shown in \cite{TSF}, normalizing flows can be extended to model thermodynamic ensembles across a range of temperatures. When trained on high-temperature simulation data, a flow can generate samples at lower temperatures by \textit{explicitly} accounting for temperature as a parameter of both the flow map and base distribution. Concretely, a change of temperature from $T_{\text{high}} \rightarrow T_{\text{low}}$ induces the following scaling relation in the output density:

\begin{equation}
    \label{steer}
    \rho^{T_{\text{low}}}_{Z}(z) \left| \frac{\partial\phi^{-1}_{T_\text{low}}}{\partial x}(x) \right| \propto \left[ \rho^{T_{\text{high}}}_{Z}(z) \left| \frac{\partial\phi^{-1}_{T_\text{high}}}{\partial x}(x) \right| \right]^{\kappa}
\end{equation}

where $\kappa = T_{\text{high}} / T_{\text{low}}$. A flow satisfying this relation is said to be \textit{temperature steerable}. Under volume preservation, the relation decouples: proportionality in the base distribution implies proportionality in the Jacobian determinants

\begin{equation}
    \rho^{T_{\text{low}}}_{Z}(z) \propto \left[ \rho^{T_{\text{high}}}_{Z}(z) \right]^{\kappa} \implies \left| \frac{\partial\phi^{-1}_{T_\text{low}}}{\partial x}(x) \right| \propto \left[ \left| \frac{\partial\phi^{-1}_{T_\text{high}}}{\partial x}(x) \right| \right]^{\kappa},
\end{equation}

a condition satisfied by adopting a Gaussian base with temperature-scaled variance, $z \sim \mathcal{N}(0, T_{\text{low}}\mathbf{I})$. SITA \textit{implicitly} accounts for temperature through these proportionality relations under the instantaneous change of variables.

\section{Independent Metropolis-Hastings with Surrogate Likelihoods}
\label{app:IMH}

Independent Metropolis-Hastings (IMH) is a special case of the Metropolis-Hastings algorithm in which the proposal distribution $p(y)$ does not depend on the current state of the Markov chain. Given a target density $\pi(x)$, a candidate $y \sim p(\cdot)$ is accepted with probability $\alpha(x, y) = \min\left(1, \frac{\pi(y)\, p(x)}{\pi(x)\, p(y)}\right)$, producing a Markov chain whose stationary distribution is $\pi$. The efficiency of the sampler is governed by how closely $p$ approximates $\pi$: when the two coincide every proposal is accepted, and the chain returns i.i.d.\ samples from the target.

Substituting the proposal density with the surrogate $q(\cdot)$ yields a modified acceptance ratio,
\begin{equation}
    \label{eqn:acc_iw}
    \hat{\alpha}(x, y) = \min\left(1,\, \frac{\pi(y)\, q(x)}{\pi(x)\, q(y)}\right) = \min\left(1,\, \frac{w(y)}{w(x)}\right),
\end{equation}
where $w(x) := \pi(x) / q(x)$ is the importance ratio. Note that candidate states continue to be drawn from the true proposal $p$. The resulting chain no longer satisfies detailed balance with respect to $\pi$; it instead admits a tilted stationary distribution $\tilde{\pi}(x) \propto \pi(x)\, p(x)/q(x)$, whose deviation from $\pi$ is governed entirely by the discrepancy between $p$ and $q$. When $q \equiv p$ the standard IMH chain is recovered.

\begin{prop}
A Markov chain that proposes from $p(\cdot)$ and accepts in accordance to the target $\pi(\cdot)$ and surrogate likelihood estimator $q(\cdot)$ admits the stationary distribution

\begin{equation}
    \tilde{\pi}(x) = \frac{1}{\tilde{Z}} \pi(x) \frac{p(x)}{q(x)}, \qquad \tilde{Z} = \int \pi(x) \frac{p(x)}{q(x)} dx
\end{equation}

\end{prop}

\begin{proof}
Let $\mu$ denote the stationary distribution we seek to characterize. By the \textit{detailed balance} relation we have

\begin{equation*}
    \mu(x) T(x \rightarrow y) = \mu(y) T(y \rightarrow x)
\end{equation*}

where the transition kernels satisfy

\begin{equation*}
    T(x \rightarrow y) = p(y) \hat{\alpha}(x, y), \qquad T(y \rightarrow x) = p(x) \hat{\alpha}(y, x)
\end{equation*}

where $\hat{\alpha}(y, x)$ is defined by equation \ref{eqn:acc_iw}.

Substituting,

\begin{align*}
    \mu(x) p(y) \hat{\alpha}(x, y) &= \mu(y) p(x) \hat{\alpha}(y, x)\\[0.5cm]
    \mu(x) p(y) \min\left(1,\, \frac{w(y)}{w(x)}\right) &= \mu(y) p(x) \min\left(1,\, \frac{w(x)}{w(y)}\right)\\
\end{align*}

Without loss of generality, assume that the $w(y) \geq w(x)$. Then

\begin{itemize}[leftmargin=2em, itemindent=2em]
    \item The left min evaluates to $\min\left(1,\, \frac{w(y)}{w(x)}\right) = 1$
    \item The right min evaluates to $\min\left(1,\, \frac{w(x)}{w(y)}\right) = \frac{w(x)}{w(y)}$
\end{itemize}

Subsequently,

\begin{align*}
    \mu(x) p(y) &= \mu(y) p(x) \frac{w(x)}{w(y)}\\[0.5cm]
    \frac{\mu(x)}{\mu(y)} &= \frac{p(x)}{p(y)} \cdot \frac{w(x)}{w(y)}\\
\end{align*}

By equality of the ratio $\mu(x) / \mu(y)$ , the stationary distribution must therefore be proportional to the quantity

\begin{equation*}
    \mu(x) \propto \pi(x) \frac{p(x)}{q(x)}
\end{equation*}

\end{proof}

\section{Architecture Details}
\label{app:arch}

\subsection{Input Representation}
\label{subsection: data featurization}
We encode molecular conformers as fully connected graphs whose node attributes reflect the atom types specified by alanine dipeptide's topology. Both the flow and EBM operate directly in Cartesian space, taking raw atomic coordinates as input.
\subsection{Flow Architecture: Geometric Vector Perceptrons}
\label{section: GVP}
The flow component of SITA is built on an E(3)-equivariant graph neural network employing geometric vector perceptrons \citet{satorras2021n} on molecular generation tasks~\citet{flowmol}.

Following~\citet{flowmol}, we incorporate architectural modifications shown to improve performance. Message passing follows the formulation of~\citet{jing2021equivariant}:
\begin{equation}
	\label{gvp_message_passing}
	(m^{(s)}_{i\to j},\,m^{(v)}_{i\to j})
	= \psi_{M}\Bigl([
		h^{(l)}_{i} :  d^{(l)}_{ij}],\;
	v_{i} : \left[\frac{x^{(l)}_{i} - x^{(l)}_{j}}{d^{(l)}_{ij}}
		\right]\Bigr)
\end{equation}

where $m^{(s)}_{i\to j}$ and $m^{(v)}_{i\to j}$ denote scalar and vector messages from node $i$ to $j$, $h_i$ collects scalar node feature, $d_{ij}$ encodes pairwise distances via a radial basis, and $x$ gives atomic positions. Full details on node and position updates appear in Appendix C of~\citet{flowmol}.

\subsection{EBM Architecture: Graphormer}
\label{section: graphormer}
For SITA's energy-based model, we employ the graphormer~\citet{ying2021transformers}, a transformer variant that has achieved strong results on molecular property prediction. The key departure from standard transformers lies in a structure-aware attention bias derived from graph topology. In the 3D setting, this bias emerges from the pairwise Euclidean distance matrix processed through an MLP.

At a high level, Graphormers interleave GNN-style operations within transformer blocks~\citet{grahformers}. Each attention head computes:
\begin{equation}
	\label{graphformer_head}
	\mathrm{head}
	= \mathrm{softmax}\!\Bigl(\frac{Q K^{\mathsf{T}}}{\sqrt{d}} + B\Bigr)\,V
\end{equation}
In the original formulation, $B$ is a learned bias matrix. Our implementation instead computes $B$ by passing the Euclidean distance matrix through an MLP, directly injecting geometric information into the attention mechanism.

\section{Training Setup}
\label{app:train_setup}

\subsection{Computational Resources}

All training experiments were ran on a single NVIDIA L40 GPU with 48 GB GDDR6 memory.

\subsection{Pre-training}

For both the flow and EBM, each was pre-trained for a total of 500 epochs with a batch size of 512. Each model was trained using separate Adam optimizers (\citet{adam_opt}) with learning rates of $1e-3$ for the flow and $5e-4$ for the EBM. Both models utilized separate reduce-on-plateau learning rate schedulers that monitored for loss convergence at a patience of 30 epochs and a reduction factor of $0.5$. An exponential moving average with a decay of 0.999 was used for each. The flow was trained on 500,000 conformations of alanine dipeptide simulated at 300K. In turn, the EBM was trained on 200,000 alanine dipeptide conformers generated by the flow model.

\subsection{Annealing Bootstrap}

Training setup for the annealing bootstrap is nearly identical to that of the pre-training phase. We highlight the training modifications during the annealing bootstrap in Table \ref{tab:boot-mod}.

\begin{table}[H]
    \caption{Bootstrap Modifications}
    \label{tab:boot-mod}
    \centering
    \begin{tabular}{lccc}
    \toprule
     Model & Epochs & Learning Rate \\
    \midrule
    EBM & 400 & $5e-4$ \\
    Flow & 200 & $5e-4$ \\
    \bottomrule
    \end{tabular}
\end{table}

\section{Metrics}
\label{app:metrics}

\paragraph{Effective Sample Size}
Reported ESS is calculated by

\begin{equation}
    \text{ESS}(\{w_i\}^{N}_{i}) = \frac{1}{\sum_i w^2_i}
\end{equation}

where weights $w_i$ are normalized.

\paragraph{Evaluation via optimal transport.} We measure sample quality through the lens of optimal transport. Let $\{x_i\}_{i=1}^n$ and $\{y_j\}_{j=1}^m$ denote samples from the generated and reference distributions, respectively. The 2-Wasserstein distance seeks the minimum-cost coupling $\pi$ between these point clouds:
\begin{equation}
    W_2(\mu, \nu) = \left( \min_{\pi \in \Pi(\mu,\nu)} \sum_{i=1}^{n} \sum_{j=1}^{m} \pi_{ij} \, c(x_i, y_j)^2 \right)^{1/2}
\end{equation}
where $\Pi(\mu,\nu)$ contains all joint distributions with the correct marginals. Optimal couplings are obtained via the POT library~\citet{flamary2021pot}. Below we describe two cost functions tailored to alanine dipeptide.

\paragraph{Energetic Cost.} Potential energy provides a global summary sensitive to both local bonded geometry and long-range interactions. We assess agreement in energy space by setting
\begin{equation}
    c_E(x,y)^2 = \bigl( E(x) - E(y) \bigr)^2
\end{equation}
yielding the metric $E$-$\mathcal{W}_2$.

\paragraph{Torsional Cost.} Ramachandran angles $(\phi, \psi)$ offer a compact descriptor of backbone geometry. For alanine dipeptide, the backbone geometry is fully characterized by two dihedral angles, $\phi$ and $\psi$. Since these angles live on a torus, we adopt a periodic cost:
\begin{equation}
    c_{\mathcal{T}}(x,y)^2 = \Bigl[ \bigl( \phi_x - \phi_y + \pi \bigr) \bmod 2\pi - \pi \Bigr]^2 + \Bigl[ \bigl( \psi_x - \psi_y + \pi \bigr) \bmod 2\pi - \pi \Bigr]^2
\end{equation}

\paragraph{Ramachandran KL divergence.} We discretize the $(\phi, \psi)$ torus into a uniform grid and estimate densities via 2D histograms over $[-\pi, \pi]^2$. The KL divergence between reference and generated distributions is then approximated as
\begin{equation}
    D_{\mathrm{KL}}(p \| q) \approx \sum_{i,j} p_{ij} \log \frac{p_{ij} + \epsilon}{q_{ij} + \epsilon} \cdot \Delta^2
\end{equation}
where $p_{ij}$ and $q_{ij}$ are the normalized histogram densities in bin $(i,j)$, $\epsilon$ is a small constant for numerical stability, and $\Delta = 2\pi / N_{\mathrm{bins}}$ is the bin width.

\section{TICA Evaluations}

\paragraph{Time-lagged Independent Component Analysis.} TICA identifies linear projections of multivariate time-series data along which the projected signal is maximally autocorrelated. In molecular dynamics, it is commonly used to recover the slow modes that separate metastable states. For a mean-centered trajectory $\tilde{x}_t \in \mathbb{R}^n$, consider a centered scalar observable $w^\top \tilde{x}_t$, formed by linearly combining the entries of a featurized trajectory $\tilde{x}_t \in \mathbb{R}^n$ through a weight vector $w$. Its autocorrelation at lag $\tau$ is given by the Rayleigh quotient

\begin{equation}
    \rho(w; \tau) \;=\; \frac{w^\top \hat{C}_{0\tau}\, w}{w^\top \hat{C}_{00}\, w},
\end{equation}
in which $\hat{C}_{00}$ and $\hat{C}_{0\tau}$ are the empirical instantaneous and lagged covariance matrices of the trajectory,
\begin{equation}
    \hat{C}_{00} \;=\; \frac{1}{T-\tau}\sum_{t=1}^{T-\tau} \tilde{x}_t \tilde{x}_t^\top,
    \qquad
    \hat{C}_{0\tau} \;=\; \frac{1}{T-\tau}\sum_{t=1}^{T-\tau} \tilde{x}_t \tilde{x}_{t+\tau}^\top.
\end{equation}
TICA selects the weight vectors that maximize $\rho(w; \tau)$. Stationarity of the Rayleigh quotient produces the generalized eigenvalue problem
\begin{equation}
    \hat{C}_{0\tau}\, w_j \;=\; \lambda_j\, \hat{C}_{00}\, w_j,
\end{equation}
whose eigenvalues $\lambda_j \in (0, 1]$ rank candidate directions by how slowly their projections lose correlation: values of $\lambda_j$ near unity correspond to dynamics that persist over many lag times, while values near zero indicate fast modes interchangeable with noise.

\paragraph{Wasserstein Metrics.} For each frame of both the reference MD trajectory and the model-generated samples, an identical featurization is applied. Heavy atoms (C, N, S) are first selected; from these we compute and concatenate (i) the upper-triangular pairwise distances and (ii) the $\sin$ and $\cos$ encodings of the backbone dihedral angles. Encoding the dihedrals via sine and cosine respects their angular periodicity in the feature representation.

A TICA model is then fit on the featurized MD trajectory at the specified lag time. Once fit, the TICA basis $\{w_j\}$ is applied to both the MD and model-sample features, yielding empirical distributions over the projections. We then compute the Wasserstein-$p$ distance between these empirical distributions by solving the discrete optimal transport problem
\begin{equation}
    \mathcal{W}_p(\hat{\mu}, \hat{\nu})^p \;=\; \min_{\Pi \in \Pi(\hat{\mu}, \hat{\nu})} \sum_{i,k} \Pi_{ik}\, \|\hat{x}_i - \hat{y}_k\|_2^p,
\end{equation}
where $\Pi(\hat{\mu}, \hat{\nu})$ is the set of couplings with the prescribed uniform marginals and the ground cost is Euclidean. Setting $p = 1$ yields $\mathcal{W}_1$ directly; for $\mathcal{W}_2$, the transport problem is solved with squared-Euclidean ground cost and the square root of the resulting optimum is returned.

\subsection{Sources of TICA Evaluation Error}
\label{app:tica}

In this section we discuss the use of time-lagged independent component analysis (TICA) \cite{noe_tica, tica_pande} for assessing generative model performance. In the application of generative modeling to MD data, it is common to quantify the discrepancy between TICA projections of model generated samples and MD simulation generated samples. In fact, the original PITA paper reports Wasserstein-1 and Wasserstein-2 metrics on such TICA projections. However, on inspection of PITA's publicly available code, we discovered two sources of error in their TICA evaluations: (i) choice of lagtimes and (ii) down-sampling trajectory frames.

\paragraph{Choice of lag time.} The sensitivity of TICA projections to the lag time $\tau$ stems directly from what the method is constructed to optimize: the projections that maximize the autocorrelation at \emph{exactly} the chosen separation, recovered as the leading eigenvectors of the generalized problem $\hat{C}_{0\tau}\, w_j \;=\; \lambda_j\, \hat{C}_{00}\, w_j$. Because $\hat{C}_{0\tau}$ encodes correlation structure at a single timescale rather than aggregating information across timescales, different choices of $\tau$ define genuinely different optimization problems whose top eigenvectors need not be related. At small $\tau$, the autocorrelation criterion rewards any motion that persists over short windows such as including high-frequency vibrations and side-chain rearrangements. Therefore the leading components mix biophysically meaningful fast processes (e.g.\ loop fluctuations or local hydrogen-bond dynamics that are genuinely correlated on picosecond scales) with thermal noise and integration artifacts that happen to be autocorrelated over a few frames. Larger $\tau$ filters these out by demanding persistence over longer windows, but at the cost of discarding fast processes that may be physically meaningful. Thus the choice of $\tau$ is fundamentally a choice about which timescale of dynamics one is willing to treat as signal. TICA metrics reported in \citet{akhound2025progressive} were computed with system-specific lag times: 10 for ATP and 100 for ADP. The small lag time chosen for ATP therefore demands a functional justification, which, to our knowledge, has not been established for ATP. We visualize TICA's sensitivity to $\tau$ for both alanine dipeptide and alanine tripeptide alongside different down-sampling techniques in Figure \ref{fig:tica_errors}.

\paragraph{Down-sampling.} Because MD trajectories are temporally correlated rather than i.i.d., the leading $N$ frames do not constitute a representative sample of the equilibrium distribution. Using them as a reference introduces systematic bias into model evaluation which can clearly be seen in Figure \ref{fig:tica_errors}. TICA metrics reported in \citet{akhound2025progressive} relied on projecting the first 10,000 frames of the MD trajectory of each molecular system studied in order to match the number of samples generated by the diffusion model at test-time. Selection of the first 10,000 frames results in the dropping of modes. Wasserstein metrics computed on biased TICA projections can therefore rank mode-collapsed models above more diverse alternatives: when the reference itself underrepresents certain modes, a model that also underrepresents them is not penalized for doing so. This pathology is apparent in \citet{akhound2025progressive}, where MD-Diff attains the best TICA-$\mathcal{W}_1$ and TICA-$\mathcal{W}_2$ values across the reported baselines despite clear mode collapse in their own Figure~5. In Table ~\ref{tab:tica-atp}, we report SITA's performance on TICA-$\mathcal{W}_1$ and TICA-$\mathcal{W}_2$ for alanine tripeptide using a lag time of 100 for completeness where we once again observe a stark reduction in both metrics when uniformly down-sampling.

\begin{figure}[!htbp]
    \centering
    \begin{subfigure}{\linewidth}
        \centering
        \includegraphics[width=\linewidth]{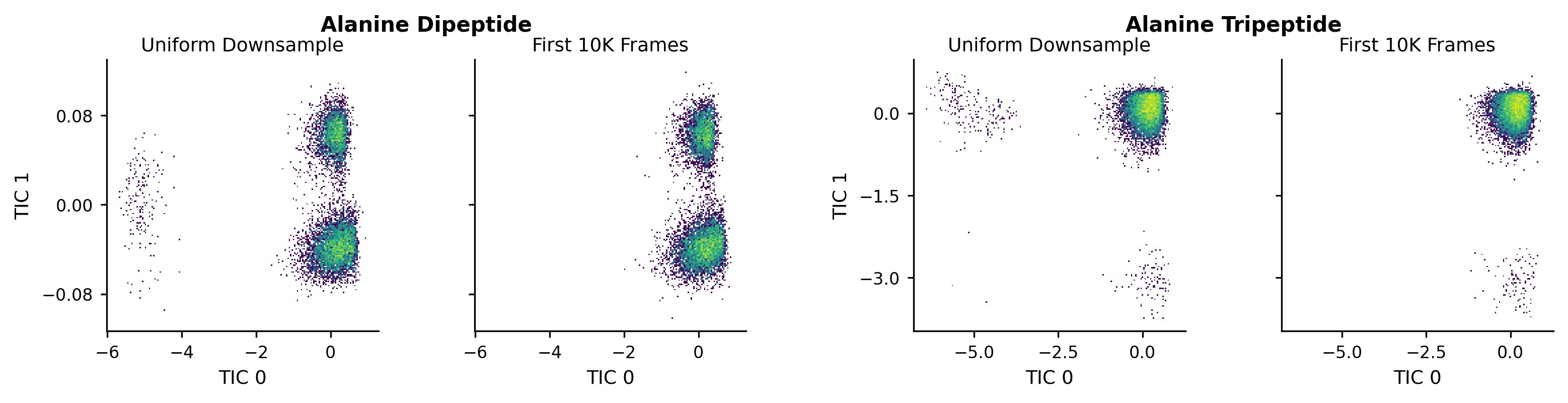}
    \end{subfigure}
    \\[1ex](a) TICA with $\tau = 10$.
    \begin{subfigure}{\linewidth}
        \centering
        \includegraphics[width=\linewidth]{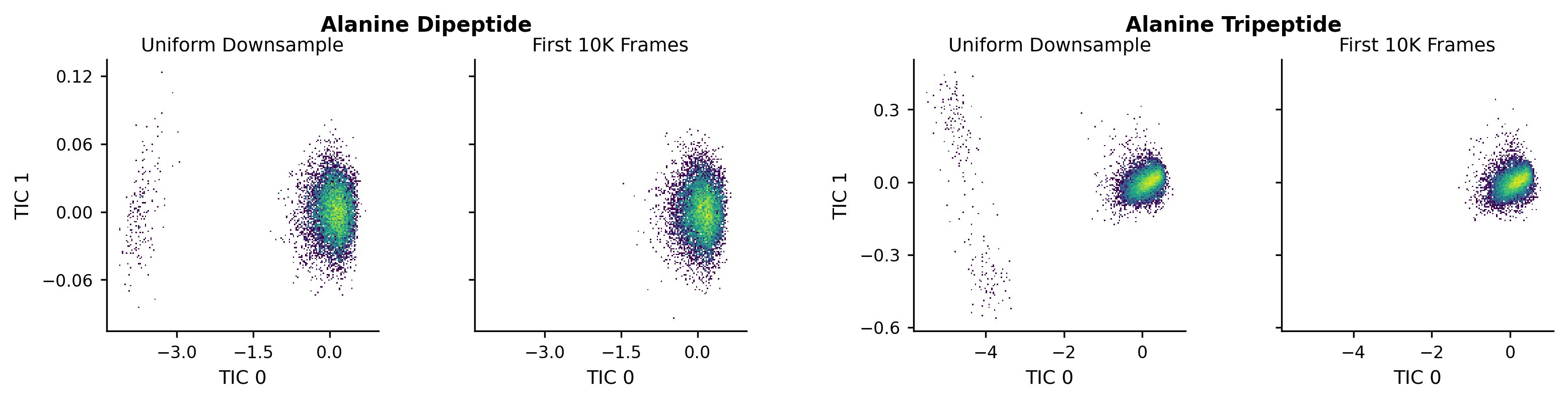}
    \end{subfigure}
    \\(b) TICA with $\tau = 100$.
    \caption{TICA down-sampling comparison at different lag times. All plots represent projections of MD derived samples.}
\label{fig:tica_errors}
\end{figure}

\begin{table}[!htbp]
\centering
\resizebox{\textwidth}{!}{%
\begin{tabular}{lcccc}
\toprule 
 & $\mathcal{W}_1$ & $\mathcal{W}_2$ & $\mathcal{W}_1^*$ & $\mathcal{W}_2^*$ \\
\midrule
SITA                  & $0.184 \pm 0.007$ & $0.723 \pm 0.012$ & $\mathbf{0.073 \pm 0.009}$ & $0.315 \pm 0.038$ \\
SITA-IS & --- & --- & $0.212 \pm 0.007$ & $0.439 \pm 0.025$ \\
SITA-IMH & --- & --- & $0.155 \pm 0.001$ & $\mathbf{0.219 \pm 0.016}$ \\
\bottomrule
\end{tabular}
}
\caption{TICA-$\mathcal{W}_1$ and TICA-$\mathcal{W}_2$ metrics for alanine tripeptide computed using a lag time of $\tau = 100$. Columns marked with $^*$ denote metrics computed using uniform frame sampling along the MD trajectory, while unmarked columns use the first-frame down-sampling method from the original PITA paper. All metrics are calculated over 10,000 samples across 3 seeds.}
\label{tab:tica-atp}
\end{table}

\section{Pseudocode}
\label{app:code}

\begin{algorithm}[H]
\caption{Independent Metropolis--Hastings Refinement with a Learned Surrogate}
\label{alg:imh}
\begin{algorithmic}[1]
\REQUIRE Target Boltzmann distribution $\pi(x)$ at temperature $T_K$ with energy function $E(\cdot)$
\REQUIRE Pre-trained flow $\hat{v}_\theta$, pre-trained EBM surrogate $q_\phi$
\REQUIRE Number of steps $K$, initialized chain $x_0 \sim \hat{\rho}_1^{T_K}$

\FOR{$i = 1$ to $K$}
    \STATE \textit{// Propose from the flow}
    \STATE Generate $y \sim \hat{\rho}_1^{T_K}$ by integrating generative ODE
    \STATE
    \STATE \textit{// Evaluate target and surrogate densities}
    \STATE Compute $\pi(x_{k-1}),\ \pi(y)$ via Boltzmann weights $\propto \exp(-E(\cdot)/k_B T_K)$
    \STATE Compute $q_\phi(x_{k-1}),\ q_\phi(y)$ from the learned surrogate
    \STATE
    \STATE \textit{// Acceptance probability}
    \STATE $\alpha(x_{k-1}, y) \leftarrow \min\!\left(1,\ \dfrac{\pi(y)\, q_\phi(x_{k-1})}{\pi(x_{k-1})\, q_\phi(y)}\right)$
    \STATE
    \STATE \textit{// Accept--reject}
    \STATE Draw $u \sim \mathcal{U}(0, 1)$
    \IF{$u \leq \alpha(x_{k-1}, y)$}
        \STATE $x_k \leftarrow y$
    \ELSE
        \STATE $x_k \leftarrow x_{k-1}$
    \ENDIF
\ENDFOR
\RETURN $x_K$
\end{algorithmic}
\end{algorithm}

\newpage
{
\small
\bibliography{sita}
}

\end{document}